\documentclass[twoside,10pt]{article}
\usepackage{aistats2016}
\usepackage{fancyhdr}
\usepackage{graphicx}
\usepackage{amsmath, amssymb, amsthm}
\usepackage[left = 1in, right = 0.75in, top = 1in, bottom = 0.75in]{geometry}

\usepackage{bbm}
\usepackage{color}
\usepackage{enumerate}
\newtheorem{theorem}{Theorem}

\newtheorem{lemma}{Lemma}

\newcommand{\zap}[1]{}

\newcommand{\betav}{\pmb{\beta}}

\newcommand{\lambdav}{\pmb{\lambda}}
\newcommand{\alphav}{\pmb{\alpha}}
\newcommand{\zerov}{\pmb{0}}

\newcommand{\wv}{\mathbf{w}}
\newcommand{\ev}{\mathbf{e}}

\usepackage{times}
\usepackage{graphicx} 
\usepackage{subfigure} 

\usepackage{natbib}

\usepackage{algorithm}
\usepackage{algorithmic}

\begin{document}
\twocolumn[

\aistatstitle{Bayes-Optimal Effort Allocation in Crowdsourcing: Bounds and Index Policies}

\aistatsauthor{ Weici Hu \\ wh343@cornell.edu \And Peter I. Frazier \\ pf98@cornell.edu}

\aistatsaddress{School of Operations Research \& Information Engineering\\  Cornell University \\ 206 Rhodes Hall\\Ithaca, NY 14853, USA } ]

\begin{abstract}
We consider effort allocation in crowdsourcing, where we wish to assign labeling tasks to imperfect homogeneous crowd workers to maximize overall accuracy in a continuous-time Bayesian setting, subject to budget and time constraints. The Bayes-optimal policy for this problem is the solution to a partially observable Markov decision process, but the curse of dimensionality renders the computation infeasible. Based on the Lagrangian Relaxation technique in \cite{adel2008}, we provide a computationally tractable instance-specific upper bound on the value of this Bayes-optimal policy, which can in turn be used to bound the optimality gap of any other sub-optimal policy. In an approach similar in spirit to the Whittle index for restless multi-armed bandits, we provide an index policy for effort allocation in crowdsourcing and demonstrate numerically that it outperforms other state-of-arts and performs close to optimal solution.
\end{abstract}

\section{Introduction}
Crowdsourcing can accomplish large-volume tasks such as image classification or document relevance assessment by using large pool of amateur workers at much less expense than is possible by hiring experts or by developing an automatic machine learning method \citep{KargerNips}.
Moreover, online platforms such as Amazon Mechanical Turk make crowdsourcing service widely accessible by providing a marketplace in which requesters may post tasks, which crowd-workers may complete in exchange for money.
These factors are making crowdsourcing increasingly important.

Although crowdsourcing is less expensive than hiring experts, the number of images or other tasks that a requester can correctly label or process is nonetheless limited by his or her budget.  This fact is compounded by the noise and variability inherent to crowd-workers' responses, which typically requires a single item to be processed independently several times by multiple workers.

In this paper, our goal is to find a sequential allocation of workers to tasks that most accurately supports a correct aggregated label for each task,
subject to a limited budget (which in turn limits the number of workers that a requester can hire) and a limited time horizon.  In this paper we focus on binary labeling tasks, but our approach can also be extended to multi-class labeling.

Intuitively, much can be accomplished through a sophisticated allocation of worker effort:
When budgets are large relative to the overall difficulty of the tasks to be accomplished, a good scheme should allocating more workers to those tasks that are more difficult, so that uniform quality can be ensured.
When budgets are small, however, those most difficult tasks should be abandoned so that the bulk of the budget can be used to ensure that at least those easy tasks are done correctly.

We adopt a Bayesian approach, which is natural in crowdsourcing because:
1) It allows us to leverage prior information about the tasks to be accomplished, which may be learned in the crowdsourcing setting from features associated with each task and the typically large collections of historical data collected in previous crowdsourcing campaigns;
2) It seeks to maximize average-case performance with respect to the prior distribution, which is natural in crowdsourcing where requesters typically tolerate some variability in quality, and are most interested in maximizing aggregate performance across a large volume of tasks, rather than ensuring robustness to some worst-case distribution over task characteristics, or studying asymptotic behaviors that do not become relevant until the number of workers working on each task grows large.

Within this Bayesian framework, we formulate and study sequential effort allocation as a partially observable Markov decision process, using tools from dynamic programming.  While the curse of dimensionality \citep{Po07} prevents solving this dynamic program to optimality, we provide a computationally tractable upper bound on the expected performance under any Bayes-optimal effort allocation policy. Upper bounds are useful because they allow evaluating the optimality gap for any given heuristic on any problem instance, simply by simulating the heuristic and comparing its performance to the bound. The technique we use to obtain such upper bound is the Lagrangian Relaxation on weakly coupled dynamic programs discussed in \cite{adel2008} and \cite{Hawkins2003}. The proofs we present in Section \ref{sec:bound} are very similar in spirit to \cite{adel2008}, but while Adelman based his proof on the value functions of the DP formulation in a infinite horizon setting, we offer a proof based on the initial objective function of the problem in a finite horizon setting. Nonetheless, our crowdsourcing model is a specific application of the more general formulation in \cite{adel2008} and \cite{Hawkins2003}. Then, using Lagrange multipliers that appear in this upper bound, we derive an index-based heuristic policy that is similar in spirit to the Gittins index policy for multi-armed bandits \cite{Gi89} and the Whittle index policy for restless bandits \cite{Wh88}. We then show that this index policy has performance close to the upper bound in numerical experiments, and also outperforms other state-of-art policies for resource allocation .

Although the primary novelty and contribution of our paper is that it is the first to characterize the performance of the Bayes-optimal policy for effort allocation in crowdsourcing, and to develop Bayesian bandit-style index policies, our work is also novel is modeling the {\it asynchronous} nature of crowd-work in a continuous-time setting, in contrast with previous work on effort allocation in crowdsourcing that assumed instant completion of tasks \citep{YanIcml2011}, \citep{qihangxichen2014}, \citep{kargerShah}. This model is inspired by how crowd-workers are employed on Amazon Mechanical Turk; allowing an asynchronous process thus gives a closer proximity to the real situations.

\section{Related Work}\label{sec:preWork}
There are two major strands of former works to which our work is related. The first is the work on effort allocation and crowd labeling. Much of this work adopts a frequentist viewpoint and focuses on error bounds for inference \cite{KargerAcm, GhoshAcm, KargerNips, ThanhAcm, hoIcml2013}. \cite{KargerAcm} proposed an allocation algorithm based on a random graph, and while its performance asymptotically order-optimal, one needs a very large number of workers to make this relevant. \cite{ThanhAcm} incorporates a limited budget, but lacks the notion of optimality. None of the work above considers a finite time horizon.
There is also work with more of a Bayesian flavor(\cite{YanIcml2011, BachrachIcml2012}). While they focused on the efficiency of allocation, they did not consider an optimal solution. Among the work that adopt a Bayesian framework, our work is similar to \cite{XichenQihangArch} in that we both form an optimal policy in the form of a stochastic dynamic program. Although they also provide a well-motivated heuristic policy, our work pushes further by deriving an upper bound based on this formulation of optimal policy.

The second strand resides in the literature of Multi-armed bandit (MAB) and stochastic dynamic programming. The formulation a Bayesian-optimal procedure as a dynamic program is considered in \cite{Lo91,Mo82}. Our use of Lagrangian relaxation is an application of the relaxation method of weakly coupled dynamic program discussed in \citep{adel2008}. The setting in this paper differs from the previous works by that only one task is to be assigned when a worker enters and the completion of task is not instant. The index-based policy proposed in this paper, which uses Lagrangian Multipliers to assign indices, draws inspiration from \cite{Wh88}.

\section{Problem Statement}\label{sec:model}
We consider a requester of crowdsourcing service with $K$ independent binary labeling tasks.
Due to a budget constraint, the requester allows a maximum of $U$ workers to work on these tasks,
and requires all work to be completed by a time horizon $T$.
We model the arrival of workers to the crowdsourcing system by a Poisson process with rate $r$.
(Our model can be generalized to non-homogeneous Poisson processes with little additional effort.)

As each worker enters the system, the requester selects one of the $K$ tasks for the worker to label.
We let $z_\ell\in\{1,\ldots,K\}$ indicate the task assigned to the $\ell^{\mathrm{th}}$ worker. (We use $[K]$ to denote $\{1,\ldots,K\}$ for the rest of the paper.)
The worker spends a random $\mathrm{Exponential}(\mu)$ amount of time on the task $x$, independent of all else, and then provides a binary label $y_\ell$.

Workers do not always give the correct label because the task may be ambiguous and thus hard to categorize, or workers may be careless or lack background information when they conduct the labeling process.
We suppose that workers are ``homogeneous" (a term used in \cite{qihangxichen2014}), and give noisy but unbiased labels.
More specifically, each task $x$ has an associated unknown value $\theta_x\in[0,1]$, which is the underlying probability that it will be labeled as positive by a worker.  The distribution of the label generated by the $\ell^{th}$ worker given $\theta_1,...\theta_K$ and $z_\ell$ is
\begin{equation}
\label{eq:likelihood}
y_\ell|\theta_{1:K},z_\ell \sim \mathrm{Bernoulli}(\theta_{z_\ell}).
\end{equation}

We set a known threshold value $d_x$, and consider the label for task $x$ being positive if $\theta_x>d_x$. We let $B = \{x : \theta_x > d_x\}$ be the set of tasks whose correct label is positive. Note $B$ is unknown as $\theta_x$ are unknown.

For analytical convenience we use the Beta distribution, which is the conjugate prior of the Bernoulli distribution, as the prior for each $\theta_x$ independent across all $x$.
\begin{equation*}
\theta_x \sim \text{Beta}(\alpha_{0,x},\beta_{0,x}).
\end{equation*}
With the assumption of this independent beta prior on each $\theta_x$, and the conditionally independent Bernoulli responses as in \eqref{eq:likelihood}, the posterior on $\theta_x$ after some number of workers have provided responses will remain beta-distributed, with first parameter equal to the sum of $\alpha_{0,x}$ and the number of positive responses, and the second parameter equal to the sum of $\beta_{0,x}$ and the number of negative responses.
In practice, one can estimate appropriate values for the parameters $\alpha_{0,x}$ and $\beta_{0,x}$ from historical data on tasks previously labeled by the crowd. We discuss this further in section \ref{sec:comp} where numerical experiments are performed.

Note the assumption of a Beta distribution can be relaxed without a great deal of difficulty, as the posterior distribution will remain in an exponential family parameterized by the number of positive and negative labels observed for the instance.
The assumption of independence cannot be easily generalized, as it is necessary for the decomposition in our Lagrangian relaxation, without which the upper bound in Section~\ref{sec:bound} much more challenging to compute.

Thus, after the worker budget $U$ has been exhausted or the time horizon $T$ has elapsed, the requester will have a posterior distribution on each $\theta_x$ which remains beta-distributed. Let $\alpha'_x$, $\beta'_x$ be the posterior parameter for this time.  At this time, we model the requester as choosing, for each task $x$, an estimated label based on the responses of the crowd-workers, and then receiving a reward of $1$ for each correctly labeled task, and $0$ for the incorrectly labeled tasks.
(Our approach can be easily generalized to other reward or loss structures that are additive across tasks, and depend only on $\theta_x$ and some task-specific estimate based on the crowd's feedback.)

The expected reward under the posterior that the requester will obtain is 
$\mathbb{P}(\theta_x > d_x | \alpha'_x,\beta'_x)$, 
if s/he chooses a positive label, 
and $\mathbb{P}(\theta_x < d_x | \alpha'_x,\beta'_x)$ if s/he chooses a negative label ($\theta_x$ has a density, and so $\theta_x=d_x$ with a posterior probability of $0$).
Thus, the requester chooses the label giving the larger reward, and achieves a reward whose expected value under the posterior is,
\begin{equation*}
R(\alpha'_x,\beta'_x) = \max\Big\{\mathbb{P}[\theta_x\!>\!d_x | \alpha'_{x},\beta'_{x} ],
\mathbb{P} [\theta_x\!<\!d_x | \alpha'_{x},\beta'_{x}]\Big\},
\end{equation*}
Across all tasks, the requester's expected reward under the posterior is
\begin{equation}
R(\alphav',\betav') = \sum_{x=1}^K R(\alpha'_x,\beta'_x),
\label{eq:maxReward}
\end{equation}
where $\alphav'=(\alpha'_x : x\in[K])$ and similarly for $\betav'$.

The goal of the requester is to design a policy to dynamically assign tasks to workers entering the system so as to maximize the expected reward received, based on the labels obtained from the crowd-workers.

\section{Dynamic Programming Formulation} \label{sec:DP}

We now formalize the problem statement from Section~\ref{sec:model} as control of a continuous-time Markov chain, which can be analyzed through a stochastic dynamic program built on the embedded discrete-time Markov chain. This continuous-time Markov chain tracks the evolution of worker assignments and posterior distributions on $\theta_x$ that results from a requester's dynamic assignment policy.

The state of this continuous-time Markov chain contains:
\begin{itemize}
\item
length-$K$ vectors $\alphav = (\alpha_1,\ldots,\alpha_K)$ and $\betav = (\beta_1,\ldots,\beta_K)$ that will describe the posterior distribution on each $\theta_x$ given the labels observed thus far ($\theta_x$ will be distributed according to $\text{Beta}(\alpha_x,\beta_x)$ under this posterior).
\item a length-$K$ vector $\wv = (w_1,...,w_K)$ that tracks the number of workers currently working on each task.
\item an integer $\ell$ that tracks the number of workers that have entered the system and been assigned to tasks (but not necessarily completed them).
\item the time $t$ of the most recent {\it event}, either a worker completing a task, or a worker arriving.
\end{itemize}
We indicate such a generic state by $s= (\alphav, \betav, t, \wv, l)$ and let $\mathbb{S} = \mathbb{R}^K\times \mathbb{R}^K\times \mathbb{R} \times \mathbb{N}^K \times \mathbb{N}$ be the set of possible values this state can take.
We let $\alphav(s)$, $\betav(s)$, $t(s)$, $\wv(s)$ and $\ell(s)$ all indicate the corresponding components of $s$.

Transitions occur in this Markov chain when workers complete tasks, and when workers arrive to start work on a task.
We use $n$ to count the number transitions (or ``events''), we let $S_n \in \mathbb{S}$ indicate the state just after the $n^{th}$ event, for $n\ge1$.   The initial state is $S_0 = (\alphav_0,\betav_0, 0, \zerov, 0)$, where $\alphav_0 = (\alpha_{0,x} : x\in[K])$ and $\betav_0 = (\beta_{0,x} : x\in[K])$ together describe the prior distribution, and $\zerov$ is a vector of $K$ zeros.

We let $\Delta_n$ denote the time duration between event $n$ and $n+1$, i.e., $\Delta_n = t(S_{n+1})-t(S_n)$.  Then, $\Delta_n|S_n \sim \text{Exp}(\mu\sum_{x=1}^Kw_x(S_n)+r)$.

We define a policy $\pi$ that controls how the requester assigns incoming workers to tasks, based on the current state.  This policy $\pi$ will map $S_n$ and $\Delta_n$ onto $\{0,1\}^K$, and $\pi(S_n,\Delta_n)$ will give the number of new workers assigned to each of the $K$ tasks, if the transition from $S_n$ to $S_{n+1}$ was caused by an arriving worker. Below we will constrain this to prevent assigning more than one task to a worker, and then later in the Lagrangian relaxation we will relax this constraint.

Formally, let $\Pi$ be the set of all measurable functions from $\mathbb{S}\times \mathbb{R}_+$ to $\{0,1\}^K$.
Then, let $|\cdot|$ return the sum of individual components of a vector, and define 
\begin{equation}
\Pi_0 = \{\pi\in \Pi: |\pi(s,\Delta)|\leq 1, \forall s\in \mathbb{S}, t\in \mathbb{R}\},
\end{equation}
where we have added the additional constraint to $\Pi$ that at most one task can be assigned to an incoming worker.
Only those $\pi\in\Pi_0$ will be feasible policies for the problem of interest, but we will consider the larger set of policies $\Pi$ to support later theoretical analysis.

The set of policies $\Pi_0$ allows not assigning an incoming worker to a task even when budget or time remains, but we will see below that this will still exhaust one unit of budget, and so optimal policies (or reasonable heuristics) will always assign incoming workers to tasks when possible.

Each $\pi\in \Pi$ defines a discrete time Markov chain $(S_n:n \in \{0,1,\ldots\})$ over the state space $\mathbb{S}$, whose transition kernel we will indicate by $\mathbb{P}^{\pi}(s'|s)$.
This transition kernel can be written as
{\small
\begin{align*}
\mathbb{P}^{\pi}(s'|s) = \int_{0}^{\infty}\!\!\mathbb{P}^{\pi}(s'|s,\Delta)\exp(-\Delta q(s))\,d\Delta,
\end{align*}
}
where we have defined
\begin{equation*}
q(s) = \mu\sum_{x=1}^Kw_x(s)+r.
\end{equation*}

Thus, to complete the description of this transition kernel, it is sufficient to describe $\mathbb{P}^{\pi}(S_{n+1}|S_{n},\Delta)$. For this description we suppose $S_n = (\alphav,\betav,t,\wv,\ell)$ and let $q = q(S_n)$.

When $t + \Delta_n\geq T$, the system has exceeded its time horizon, all outstanding tasks on which workers are currently working are canceled, and only the time is updated:
$S_{n+1} = (\alphav,\betav,t+\Delta_n,\zerov,\ell)$.


When $t + \Delta_n < T$, time remains and the next event can be either a worker arrival or a worker completion. A completion either outputs a positive result or a negative result.

A worker arrives with probability $r/q$.  If $\ell < U$, then the requester allocates this worker to a task, and the total number of arrivals is incremented:
$S_{n+1} = (\alphav, \betav, t+\Delta_n, \wv + \pi(S_n,\Delta_n), \ell+1)$.
If $\ell \ge U$, then the worker budget has been exceeded, and the requester cannot allocate the worker, so
$S_{n+1} = (\alphav, \betav, t+\Delta_n, \wv,\Delta_n), \ell)$.

For each $x\in[K]$, a worker completes this task $x$ and reports a positive label with probability 
$\frac{\alpha_x}{\alpha_x+\beta_x}\frac{\mu w_x}{q}$.
When this occurs, $S_{n+1} = (\alphav+\ev_x, \betav, t, \wv - \ev_x, \ell)$.

Similarly, a worker completes task $x$ and reports a negative label with probability 
$\frac{\beta_x}{\alpha_x+\beta_x}\frac{\mu w_x}{q}$.
When this occurs, $S_{n+1} = (\alphav, \betav + \ev_x, t, \wv - \ev_x, \ell)$.

This completely specifies the transition kernel for the discrete-time Markov chain that describes the continuous-time dynamics of both worker allocation and the posterior distribution on each $\theta_x$.

To model completion, we then define $\mathbb{S}_A = \{s\in \mathbb{S} :  \text{$t(s) \ge T$ or ($\ell(s)\ge U$ and $\wv(s)=\zerov$)}\}$ to be the set of states in which our time horizon has elapsed, or our worker budget has been exhausted and all allocated workers have finished their work. We then let $N = \inf\{n\geq 0: S_n \in \mathbb{S}_A \}$ be the number of events that occur up to and including the time when we reach a state in $\mathbb{S}_A$.
The posterior $\alphav(S_N),\betav(S_N)$ is the one with which the requester must make his/her final determination of the task labels, and so the expected reward under the posterior that s/he receives at time $t(S_N)$ is $R(\alphav(S_N),\betav(S_N))$.

Recall that our goal stated in section~\ref{sec:model} was to find the dynamic allocation policy $\pi$ of workers to tasks that maximizes the expected number of correctly classified tasks.  With the definition of this Markov chain in place, this overarching goal may be stated formally as solving
\begin{equation}\label{maxR}
\sup_{\pi\in\Pi_0}\mathbb{E}^{\pi}\left[R(\alphav(S_N),\betav(S_N)) \right].
\end{equation}

As a stochastic control problem, its solution may be characterized using stochastic dynamic programming.
We define the value function as 
\begin{equation}
V(s) = 
\sup_{\pi\in\Pi_0}\mathbb{E}^{\pi}\left[R(\alphav(S_N),\betav(S_N))| S_0 = s \right],
\end{equation}
and observe that the value function satisfies the dynamic programming recursion.  

First, for $s\in \mathbb{S}_A$, we have $V(s) = R(\alphav(s),\betav(s))$.
Then, for $s=(\alphav,\betav,t,\wv,\ell)\notin\mathbb{S}_A$ and $q=q(s)$, we have,
If $\ell < U$:
{\footnotesize
\begin{align}\label{bellman3a}
 V(s)  =  \Big(1-\text{exp}(-q&(T-t))\Big)\cdot \Big[\hspace{30mm}\text{ }\nonumber \\
r\int_{0}^{T-t}\!\max_zV (\alphav,\betav,& t\!+\!y,\wv+\ev_{z},\ell\!+1)e^{-qy}dy+ \nonumber \\
\sum_{x=1}^K\mu w_x \Big(\frac{\alpha_x}{\alpha_x+\beta_x} 
&\int_{0}^{T-t}V(\alphav\!+\!\ev_x,\beta,t\!+\!y,\wv\!-\!\ev_x,\ell)e^{-qy}dy\hspace{10mm}\text{ }\nonumber \\
+ \frac{\beta_x}{\alpha_x+\beta_x} 
\int_{0}^{T-t} & V(\alphav,\beta\!+\!\ev_x,t\!+\!y,\wv\!-\!\ev_x,\ell)e^{-qy}dy\Big)\Big]
\nonumber\\
+  \text{exp}(-q(T-t))&\Big\{R(\alphav,\betav)\Big\}.
\end{align}
}
If $\ell \geq U$:
{\footnotesize
\begin{align}\label{bellman3b}
 V(s)  =  \Big(1-\text{exp}(-q&(T-t))\Big)\cdot \Big[\hspace{30mm}\text{ }\nonumber \\
r\int_{0}^{T-t}\!V (\alphav,\betav,& t+y,\wv,\ell)e^{-qy}dy + \nonumber \\
\sum_{x=1}^K\mu w_x \Big(\frac{\alpha_x}{\alpha_x+\beta_x} 
&\int_{0}^{T-t}V(\alphav\!+\!\ev_x,\beta,t\!+\!y,\wv\!-\!\ev_x,\ell)e^{-qy}dy\hspace{10mm}\text{ }\nonumber \\
+ \frac{\beta_x}{\alpha_x+\beta_x} 
\int_{0}^{T-t} & V(\alphav,\beta\!+\!\ev_x,t\!+\!y,\wv\!-\!\ev_x,\ell)e^{-qy}dy\Big)\Big]
\nonumber\\
+  \text{exp}(-q(T-t))&\Big\{R(\alphav,\betav)\Big\}.
\end{align}
}
Moreover, knowing the value function reveals an optimal policy: an optimal policy is given by choosing the task $Z_\ell$ to assign to the next worker, in response to previous state $S_n$ at time $t(S_n)+\Delta_n$, to achieve the maximum in $\max_z V (\alphav(S_n),\betav(S_n),t(S_n)+\Delta_n,\wv(S_n)+\ev_z,\ell+1)$.

However, solving this dynamic program is computationally infeasible. For example, if we discretize the continuous time line to just 1000 intervals, when we have $K = 4$ tasks, the number of states to consider after $l = 20$ workers entering the system is $2.26*10^{11}$, which is too big to compute. Hence we seek to first provide an upper bound to the optimal value and then use the upper bound as the yardstick to measure how close a heuristic policy performs to an optimal policy.

\section{Upper Bound on the Bayes-Optimal Policy}
\label{sec:bound}
Although solving \eqref{maxR} directly using the stochastic dynamic program \eqref{bellman3a},\eqref{bellman3b} is computationally intractable, in this section we show how to obtain a computationally feasible upper bound on the value \eqref{maxR} using a Lagrangian relaxation. 

Recall we use $n$ to count events and $S_n$ is the state corresponding to the $n^{th}$ event. Define $n_\ell$ as the number of events that have occurred by the time of the $\ell^{th}$ arrival (inclusive), i.e., $n_\ell = \inf\{n: \ell(S_n) = \ell\}$. For $1\leq \ell \leq U$, define $a_\ell =\pi(S_{n_\ell-1},\Delta_{n_\ell-1})$, so that $a_{\ell,x}=1$ if the $\ell^{th}$ worker is assigned to task $x$. Therefore $\Pi_0$ satisfies $\Pi_0 = \{\pi\in \Pi: \mathbb{P}^{\pi}(\sum_{x=1}^K a_{\ell,x} \leq 1)=1 \ \forall \ell\}$. We also define a new subset of $\Pi$:
\begin{equation}\label{pi1}
\Pi_1 = \left\{\pi\in \Pi: \mathbb{E}^{\pi}\left[\sum_{x=1}^K a_{\ell,x}\right]\leq 1 \ \forall \ell\right\}.
\end{equation}
Under $\Pi_1$, we may assign a worker to more than one task along a particular sample path, as long as the {\it expected} number of tasks assigned to each worker is no larger than $1$.  Returning to our Markov chain model, we will observe that when a worker is assigned more than one task, the tasks are completed independently from each other. Observe that $\Pi_1$ includes a larger set of policies than $\Pi_0$, and that $\Pi$ includes a set that is larger still, i.e., $\Pi_0 \subseteq \Pi_1 \subseteq \Pi$.  Our result will use this relation.

To streamline notation, let $R = R(\alphav(S_N),\betav(S_N))$.  The optimal reward for the original crowdsourcing problem \eqref{maxR} is then,
\begin{equation}
R_0 = \sup_{\pi\in \Pi_0}\mathbb{E}^{\pi}\left[R\right],
\end{equation}
and the optimal reward under the larger class of policies $\Pi_1$ is
\begin{equation}
R_1 = \sup_{\pi\in \Pi_1}\mathbb{E}^{\pi}\left[R\right].
\end{equation}

Here we introduce a non-negative vector $\lambdav = \{\lambda_1,...,\lambda_U\}\geq \mathbf{0}$, which we use below as a Lagrange multiplier within a Lagrangian relaxation.  

More specifically, we will relax the constraint that each worker is assigned to at most one task, but will penalize the number of tasks assigned in a way that ensures that an upper bound holds regardless of what $\lambdav$ is (as long as it is componentwise-nonnegative).  This will then provide an upper bound on the optimal value of the original problem $R_0$, which can be made tighter by minimizing over $\lambdav$.  This upper bound can then be computed below via decomposition into $K$ small dynamic programs of fixed dimension that can be solved efficiently, even as $K$ grows large.

Our upper bound is provided in the following theorem.
\begin{theorem}\label{lemma1}
The term
\begin{equation}\label{uppB}
\inf_{\lambdav \geq \mathbf{0}} \sup_{\pi\in\Pi}\mathbb{E}^{\pi}\Big[R - \sum_{\ell=1}^U(\lambda_\ell\sum_xa_{\ell,x})\Big] + \sum_{\ell=1}^U \lambda_\ell 
\end{equation}
forms an upper bound to $R_0$.
\end{theorem}
\begin{proof} 
\begin{align}
& \sup_{\pi\in\Pi}\mathbb{E}^{\pi}\Big[R - \sum_{\ell=1}^U(\lambda_\ell\sum_xa_{\ell,x})\Big] + \sum_{\ell=1}^U \lambda_\ell \nonumber \\
= & \sup_{\pi\in\Pi}\mathbb{E}^{\pi}\left[R - \sum_{\ell=1}^U\left(\lambda_\ell(\sum_{x}a_{\ell,x}-1)\right)\right] \nonumber \\
\geq & \sup_{\pi\in\Pi_1}\mathbb{E}^{\pi}\left[R - \sum_{\ell=1}^U\left(\lambda_\ell(\sum_{x}a_{\ell,x}-1)\right)\right] \nonumber \\
= & \sup_{\pi\in\Pi_1}\mathbb{E}^{\pi}\left[R \right] - \sum_{\ell=1}^U\lambda_\ell \left( \mathbb{E}^{\pi}[\sum_x a_{\ell,x}]-1 \right) \nonumber \\
\geq &\sup_{\pi\in\Pi_1} \mathbb{E}^{\pi}\left[R\right] \nonumber \\
\geq &\sup_{\pi\in\Pi_0} \mathbb{E}^{\pi}\left[R\right] \label{relaxp}
\end{align}
The first inequality is due to $\Pi_1 \subseteq \Pi$. The second inequality is because $\lambdav\geq \mathbf{0}$ and $\mathbb{E}^{\pi}[\sum_x a_{\ell,x}] \leq 1$ for any $\pi\in\Pi_1$. The third inequality is due to $\Pi_0 \subseteq \Pi_1$. Since \eqref{relaxp} holds true for any value of $\lambdav>\mathbf{0}$, we obtain Theorem~\ref{lemma1}.
\end{proof}

Calculating the supremum term in Theorem~\ref{lemma1} directly by dynamic programming is again computationally infeasible, because the state space of this dynamic program again is over all of $\mathbb{S}$, which has $3K+1$ dimensions.
We avoid this issue by decomposing this supremum term into the sum of the optimal values for $K$ dynamic programs, one for each task, each of which has a much more manageable $4$ dimensions.

To support this decomposition, we write the state $S_n\in\mathbb{S}$ for the whole system ($K$ tasks) as $S_n=(S_{n,1},\ldots,S_{n,K})$, where $S_{n,x}$ is the state for task $x$ when $n$ events have occurred, and includes $\alpha_x,\beta_x,t,w_x$, and the global counter $\ell$.  
We let $\mathbb{S}^{(x)} = \mathbb{R} \times \mathbb{R} \times \mathbb{R} \times \mathbb{N} \times \mathbb{N}$ be the set of possible values for this single-task state $S_{n,x}$.

Following a development identical to that in Section~\ref{sec:DP}, but for the single task $x$, we may define a space of policies $\pi^{(x)}$ that map the single-task state $S_{n,x}$ and the elapsed time since the last event $\Delta_{n}^{(x)}$ (counting worker arrivals over the whole system, and completions of task $x$ only) onto a binary decision of whether or not to allocate an incoming worker to task $x$,  so that $\pi^{(x)}(S_n, \Delta_n) \in \{0,1\}$.  Following this development, we construct $K$ independent Markov chains, one for each task, where each one is controlled by its respective single-task policy $\pi^{(x)}$.
We define $N$ as before, to be the first time that the time horizon elapses, or our worker budget has been exhausted and all outstanding workers have completed their work.
We then let $R_x = R(\alpha_x(S_{N}),\beta_x(S_N))$ be the reward obtained from this the single task at this time.

The following theorem shows that the bound in Theorem~\ref{lemma1} can be re-written in terms of the sum of solutions of single-task dynamic programming problems, where each obtains the reward $R_x$, and is penalized for assigning workers to its task.

\begin{theorem}
\label{theorem2}
\begin{equation}
\inf_{\lambdav\geq \mathbf{0}}\sum_{x=1}^K\sup_{\pi^{(x)}\in \Pi^{(x)}}\mathbb{E}^{\pi^{(x)}}\Big[R_x - \sum_{\ell=1}^U \lambda_\ell a_{\ell,x}\Big]+\sum_{\ell=1}^U \lambda_\ell
\end{equation}
forms an upper bound on $R_0$.
\end{theorem}
\begin{proof}
For any $\lambdav\geq \mathbf{0}$:
\begin{align}
& \sup_{\pi\in\Pi}\mathbb{E}^{\pi}[R - \sum_{\ell=1}^U(\lambda_\ell\sum_x a_{\ell,x})]\nonumber \\
= & \sup_{\pi\in\Pi}\mathbb{E}^{\pi}[\sum_{x=1}^K R_x - \sum_{x=1}^K\sum_{\ell=1}^U \lambda_\ell a_{\ell,x}]\nonumber \\
= & \sup_{\pi\in\Pi}\mathbb{E}^{\pi}\Big[\sum_{x=1}^K \Big(R_x - \sum_{\ell=1}^U \lambda_\ell a_{\ell,x}\Big)\Big]\nonumber\\ 
= & \sup_{\pi\in\Pi}\sum_{x=1}^K \mathbb{E}^{\pi}\Big[R_x - \sum_{\ell=1}^U \lambda_\ell a_{\ell,x}\Big]\nonumber.
\end{align}
This is bounded above by,
\begin{align}\label{decom}
&\sum_{x=1}^K\sup_{\pi\in \Pi}\mathbb{E}^{\pi}\Big[R_x - \sum_{\ell=1}^U \lambda_\ell a_{\ell,x}\Big]\nonumber\\
= & \sum_{x=1}^K\sup_{\pi^{(x)}\in \Pi^{(x)}}\mathbb{E}^{\pi^{(x)}}\Big[R_x - \sum_{\ell=1}^U \lambda_\ell a_{\ell,x}\Big].
\end{align}
The equality at (\ref{decom}) is because $\sup_{\pi\in \Pi}\mathbb{E}^{\pi}\Big[R_x - \sum_{\ell=1}^U \lambda_\ell a_{\ell,x}\Big]$ depends only on $(\alpha_{t,x},\beta_{t,x},w_{t,x} : 0 \leq t \leq T)$, which is in turn governed by $\pi^{(x)}$. 
By Theorem~\ref{lemma1}, for any $\lambdav \geq \mathbf{0}$,
\begin{equation*}
\sum_{x=1}^K\sup_{\pi^{(x)}\in \Pi^{(x)}}\mathbb{E}^{\pi^{(x)}}\Big[R_x - \sum_{\ell=1}^U \lambda_\ell a_{\ell,x}\Big] + \sum_{\ell=1}^U \lambda_\ell
\end{equation*}
forms an upper bound on $R_0$. This hold for any $\lambdav\geq \mathbf{0}$, and so we have thus proved Theorem~\ref{theorem2}. 
\end{proof}

Since the state space is much smaller for a single-task system, we can use dynamic programming to solve for the supremum term 
\begin{equation}\label{singleR}
\sup_{\pi^{(x)}\in \Pi^{(x)}}\mathbb{E}^{\pi^{(x)}}\Big[R_x - \sum_{\ell=1}^U \lambda_\ell a_{\ell,x}\Big],
\end{equation}  
for any $\lambdav$ value. What remains in the computing of the upper bound is to solve for the infimum in Theorem \ref{theorem2}. We explore the convexity property of the problem follow by a binary search. Define $B(\lambda) = \sum_{x=1}^K\sup_{\pi^{(x)}\in \Pi^{(x)}}\mathbb{E}_0^{\pi^{(x)}}\Big[R_x - \sum_{\ell=1}^U \lambda_\ell a_{\ell,x}\Big]+\sum_{\ell=1}^U \lambda_\ell$, which is the upper bound derived in Theorem \ref{theorem2} without the infimum. First we prove that $B(\lambdav)$ is convex in $\lambdav$.
\begin{lemma}
$B(\lambdav)$ is convex in $\lambdav$.
\end{lemma} 
\begin{proof}
First note $\sum_{\ell=1}^U \lambda_\ell$ is convex in $\lambdav$. To prove $\sup_{\pi^{(x)}\in \Pi^{(x)}}\mathbb{E}^{\pi^{(x)}}\Big[R_x - \sum_{l=1}^U \lambda_l a_{l,x}\Big]$ is convex in $\lambdav$, pick any $\lambdav_1,\lambdav_2 \geq \mathbf{0}$ and $t \in [0,1]$ . Let 
\begin{equation}\label{suppi}
\pi' = \sup_{\pi^{(x)}\in \Pi^{(x)}}\mathbb{E}^{\pi^{(x)}}\Big[R_x - \sum_{l=1}^U (t\lambda_{1,l}+(1-t)
\lambda_{2,l}) a_{l,x} \Big]
\end{equation}
We have
\begin{align*}
& t \sup_{\pi^{(x)}\in \Pi^{(x)}}\mathbb{E}^{\pi^{(x)}}\Big[R_x - \sum_{l=1}^U \lambda_{1,l} a_{l,x}\Big]\\
+&(1-t) \sup_{\pi^{(x)}\in \Pi^{(x)}}\mathbb{E}^{\pi^{(x)}}\Big[R_x - \sum_{l=1}^U \lambda_{2,l} a_{l,x}\Big]\\
\geq & t\mathbb{E}^{\pi'}\Big[R_x - \sum_{l=1}^U \lambda_{1,l} a_{l,x}\Big]+(1-t)\mathbb{E}^{\pi'}\Big[R_x - \sum_{l=1}^U \lambda_{2,l} a_{l,x}\Big]\\
=&\mathbb{E}^{\pi'}\Big[R_x - \sum_{l=1}^U (t\lambda_{1,l}+(1-t)\lambda_{2,l}) a_{l,x}\Big]\\
=&\sup_{\pi^{(x)}\in \Pi^{(x)}}\mathbb{E}^{\pi^{(x)}}\Big[R_x - \sum_{l=1}^U (t\lambda_{1,l}+(1-t)
\lambda_{2,l}) a_{l,x} \Big].
\end{align*}
Hence $\sup_{\pi^{(x)}\in \Pi^{(x)}}\mathbb{E}^{\pi^{(x)}}\Big[R_x - \sum_{l=1}^U \lambda_l a_{l,x}\Big]$ is convex in $\lambdav$ for any $x\in [K]$, subsequently the sum of $\sup_{\pi^{(x)}\in \Pi^{(x)}}\mathbb{E}^{\pi^{(x)}}\Big[R_x - \sum_{l=1}^U \lambda_l a_{l,x}\Big]$ across $x$ is also convex in $\lambdav$. Thus we complete the proof that $B(\lambdav)$ is convex in $\lambdav$.
\end{proof}
With the convexity in $\lambdav$, we approximate $\lambdav'$ that achieves the infimum by setting $\lambdav' = \lambda'\times (1,,,1)$, and use a Fibonacci search to find the infimum. Here we constrain all the units of $\lambdav'$ to be the same for simpler computation, a tighter bound can be obtained by allowing each unit of $\lambdav'$ to vary and use sub-gradient descent to locate the infimum.

\section{Index Policy}\label{sec:index}
We introduce an index-based heuristic policy built on the Lagrangian relaxation we used in proving the upper bound. In this policy, we compute some $\lambda_x^*$ for each task $x$ based on its state $S_{n,x}$, such that $\lambda_x^*$ is the greatest value of $\lambda$ that the optimal policy will decide to hire the worker on state $S_{n,x}$ when solving \eqref{singleR} with $\lambdav = \lambda \mathbf{1}$, $\mathbf{1} = (1,\ldots,1)$. We then assign the incoming worker to the task with the highest $\lambda^*$. The intuition behind this policy is that in a single-task problem described by \eqref{singleR}, we view $\lambda_\ell$ as a cost of employing the $\ell^{th}$ worker. As $\lambda_\ell$ increases, our decision switches from hiring the worker to not hiring, where the switching point is at $\lambda^*_x$. Hence tasks with a high $\lambda^*_x$ are the tasks that are worth hiring more workers to work on. Below we present the algorithm in a more formal way.
\begin{algorithm}
\caption{Index Policy}
\begin{algorithmic}[1]
\WHILE{$\ell<U$}
\STATE For each $x \in \{1,\ldots,K\}$, compute $\lambda^*_x = \inf \{ \lambda \in \mathbb{R}_+: a_{\ell,x}^{ \lambda }= 1\}$, where $a_{x,\ell}^\lambda$ is the optimal decision from \eqref{singleR} when $\lambdav=\lambda\mathbf{1}$.
\STATE Let $x_* = \arg\max_x{\lambda^*_x}$. Break tie arbitrarily.
\STATE Assign task $x_*$ to the $\ell^{th}$ worker.
\ENDWHILE
\end{algorithmic}
\end{algorithm}
A useful technique to reduce the amount of computation is to put a cap on the total number of workers that can be assigned to a task, for this reduces the size of the state space of the dynamic program involved in solving for \eqref{singleR}. This additional cap does not affect the decision made by the Index Policy. Intuitively it is unlikely for any reasonable policy to assign all the $U$ number of workers to one task, so it is unlike for any task to get more than a certain number of workers assigned. One can check the validity of the cap by running simulations with the capped index policy, and see whether there are tasks that use all the workers that the cap allows. 
  
We show in section \ref{sec:comp} that this policy's performance is close to optimal.

\section{Numerical Experiment}\label{sec:comp}
For numerical experiment we concentrate on the case in which $T = \infty$. In this case we stop when we reach the maximum number of workers the budget allows. The Bellman's recursion to compute \ref{singleR} in the computation of upper bound for this special case is given in the supplement. In the first set of simulation study, we compare the performances of different policies on simulated data against the corresponding upper bounds. In the second set of simulation study, we use a real dataset for simulation.

\subsection{Simulation using simulated data}\label{simSim}
In the first set of simulations we evaluate the performance of the Index Policy using simulated data, and compared the total reward given in \eqref{eq:maxReward} generated by the Index Policy to the upper bound \ref{uppB}. We also compare the performance of the Index Policy to Optimistic Knowledge Gradient(OKG) method \citep{XichenQihangArch}, which is a state-of-art Bayesian allocation policy.  A round of simulated process includes generating either an arrival of worker or a completion of task based on the arrival rate $r=0.1$ and and completion rate $\mu=0.4$ with distributions specified in Section \ref{sec:model}. If it is a completion of task, we generate a label based on the posterior parameters. The process stops when we exhaust all the budget, i.e., the number of workers that are allowed to hire, and we get a reward as in \eqref{eq:maxReward}. We vary the number of tasks to be $K = 10,100,1000$, and set the budget to be $U = 1.2K$. We use a non-informative prior with $\alphav = \mathbf{1}$ and $\betav = \mathbf{1}$. We use a threshold $d_x = 0.5$ for all the tasks. For each value of $K = 10,100,1000$, we simulate the process 5000 times, and obtain a 95\% confidence interval for the simulated total reward. In Figure \ref{fig1} we show a Semi-log plot of the number of tasks $K$ against the average reward per task with the corresponding confidence intervals.
\begin{figure}[h!]
  \centering
      \includegraphics[width=0.35\textwidth]{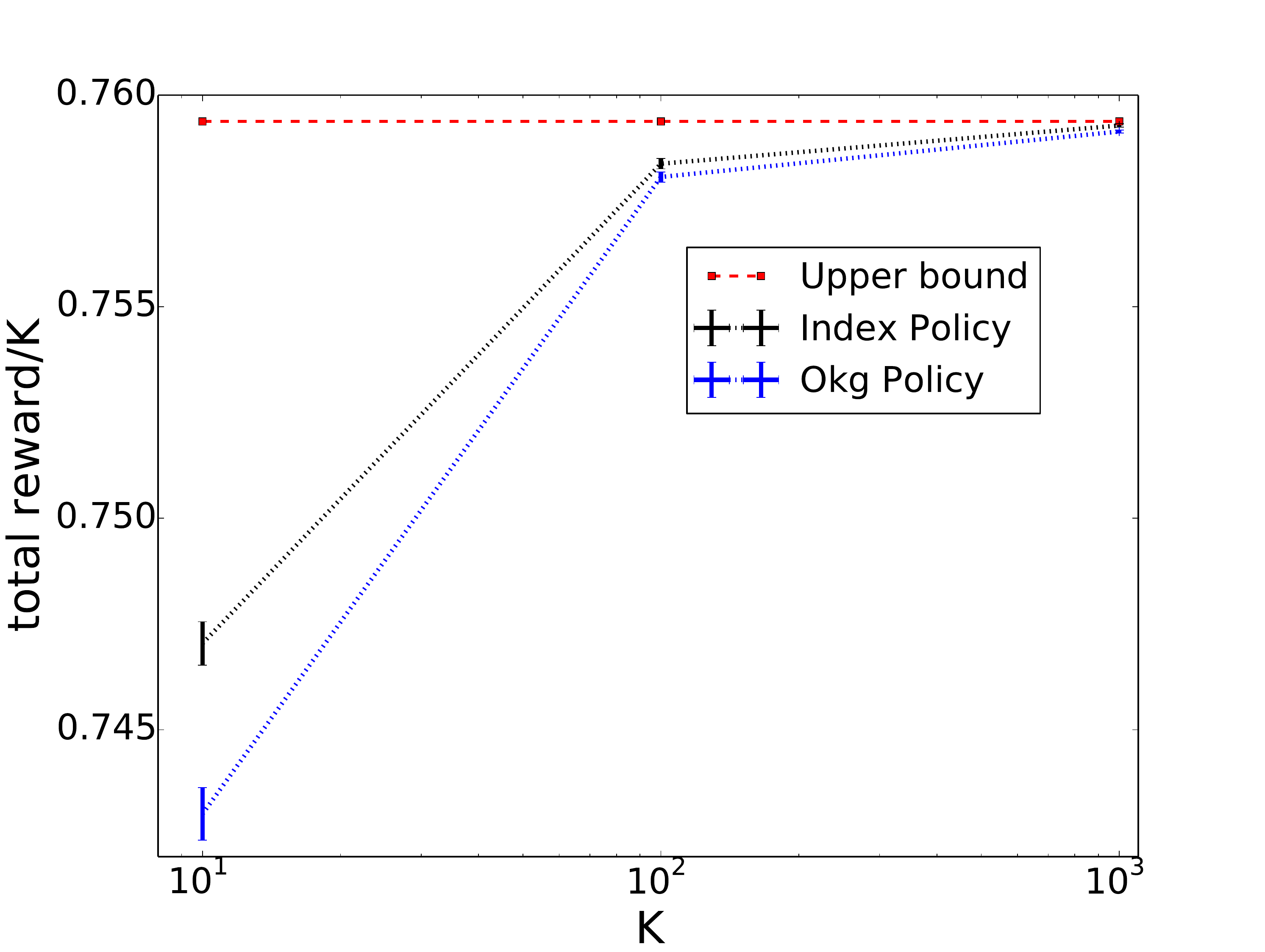}
  \caption{{\small Semi-log plot of $K$ against average per task reward ($R/K$) for $K = 10, 10^2, 10^3$.}}
\label{fig1}
\end{figure}

The performance of the index policy is consistently better than the OKG policy, especially for smaller number of tasks. Moreover, the gap between the upper bound and the simulated value gets smaller as $K$ increases, which demonstrates both that the upper bound is tight and the Index Policy performs close to an optimal policy as the number of tasks gets larger. 

We emphasize that the improved performance over OKG offered by our Index Policy is only one aspect of the contribution of this work. The other aspect is the tightness of the upper bound, especially for problem instances with many tasks. This tight upper bound for K=1200 shows that the index policy is within 0.03\% of optimal, and that continued algorithmic development will not provide significantly increased performance for large-scale crowdsourcing problems with characteristics matching this particular simulated dataset.  The ability to bound the improvement from continued algorithmic development for a particular problem instance, or class of problem instances, is useful for managers at companies that use crowdsourcing and wish to allocate engineering$\backslash$R$\&$D effort.

\subsection{Simulation using real data}
This set of simulations uses a real dataset, PASCAL RTE-1\citep{Snow2008}, which consists of 800 tasks, each comes with 10 labels obtained from crowdworkers and a gold standard label. (A gold standard label of a classification task is its true label.) We evaluate the performance of the Index Policy against the OKG policy, the Thompson Sampling \citep{Chapelle2011} and a widely used frequentist approach - the upper confidence bound(UCB) policy \citep{Auer2002}. (The specific version of UCB used here is UCB1-tuned.) The metric used to evaluate the performance is the accuracy score. More specifically, it is the number of correctly predicted labels over the total number of tasks. In each round of simulation process, we still simulate events (either arrival or completion) the same way as we did in Section \ref{simSim}. If the event is a completion, we read the most recent label for that task in the dataset. At the end of each process, predicted labels are compared against the gold standard labels. $d_x$ is still 0.5 for all tasks. For each value of $K = 10,100,750$, we simulate the process 5000 times, and obtain a 95\% confidence interval for the simulated total reward. 

Before simulating, we use the remaining 50 tasks from the RTE dataset as the `historical data' to estimate the parameters of the Beta prior, which is set to be the same across all tasks. This comes with an assumption that all tasks are homogeneous, hence a subset of them are representative of a larger population. We first obtain an estimate of $\theta_x$ for each of the 45 tasks, then use these empirical values of $\theta_x$ to fit a Beta distribution by Method of moments.
\begin{figure}[h!]
  \centering
      \includegraphics[width=0.4\textwidth]{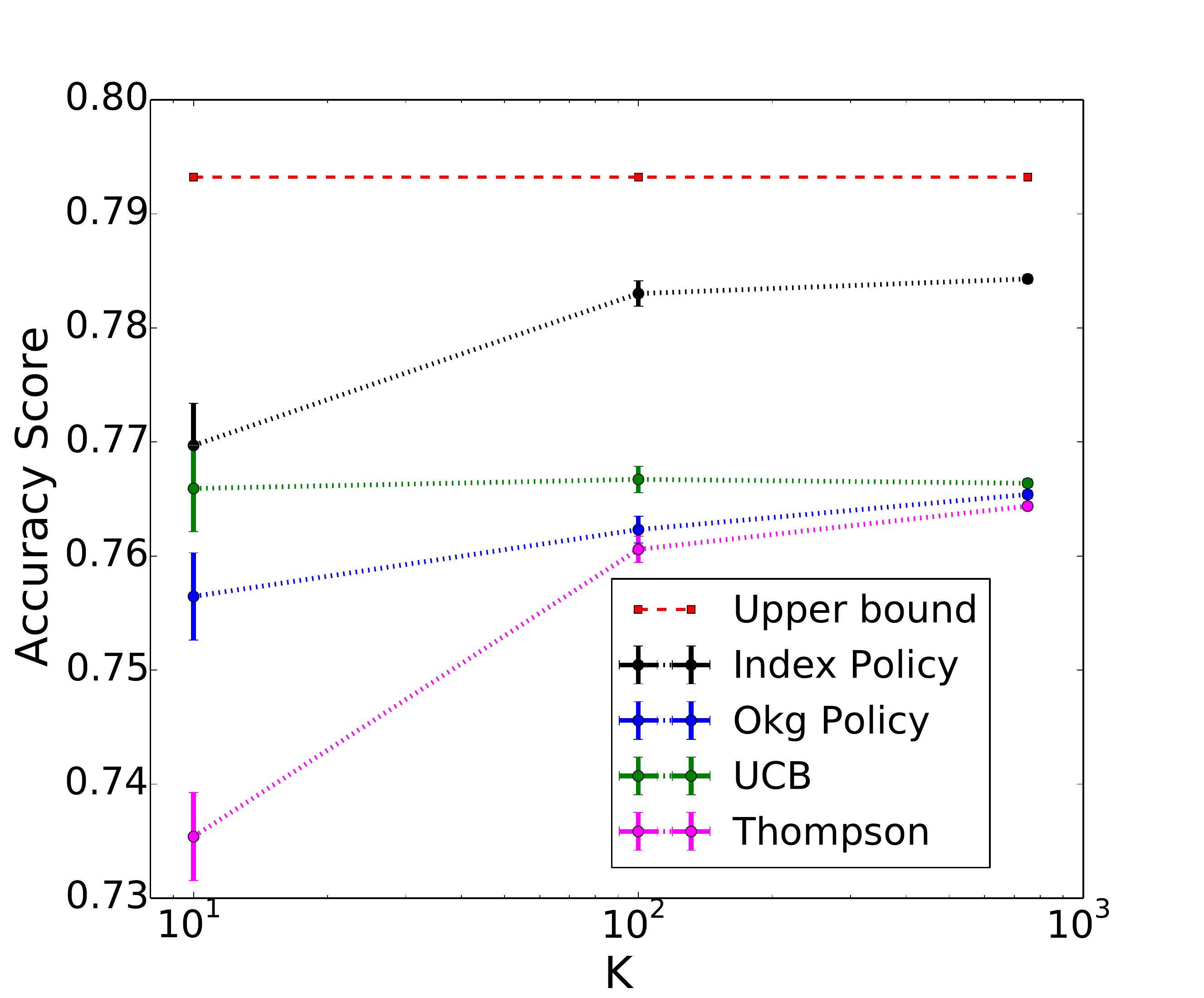}
  \caption{{\small Semi-log plot of $K$ against accuracy score for $K = 10, 100, 750$.}} 
\label{fig2}
\end{figure}
In Figure \ref{fig2} we show a Semi-log plot of the number of tasks $K$ against the accuracy score with the corresponding confidence intervals. All the Bayesian policies see an smaller optimality gap when $K$ gets larger. Index Policy performing consistently the best among all the policies. It is proven in \cite{XichenQihangArch} that OKG policy is consistent: it achieve 100\% accuracy almost surely when number of workers goes to infinity. We demonstrate numerically that the Index Policy performs better than the OKG policy. It is thus reasonable to anticipate that the Index Policy is not only consistent, but is asymptotically optimal when both the number of workers and the number of tasks goes to infinity, while keeping the ratio of the number of workers and the number of tasks constant.

\section{Conclusion}
\label{sec:conclusion}
We formulated the effort-allocation problem in crowdsourcing in a continuous time setting with budget constraint and time constraint. We also provide a computationally feasible upper bound on value of the Bayes-optimal policy using 
Lagrangian relaxation. Using the Lagrange multiplier used in proving the upper bound, we also derived an index-based policy and showed in numerical experiments that it performs close to optimal.

\section{Acknowledgment}
The authors were partially supported by NSF CAREER CMMI-1254298, NSF IIS-1247696, NSF CMMI-1536895, AFOSR FA9550-12-1-0200, AFOSR FA9550-15-1-0038, and the ACSF AVF (Atkinson Center for a Sustainable Future Academic Venture Fund).

\bibliographystyle{icml2015}
\bibliography{Reference}

\end{document}